\documentclass[10pt,twocolumn,letterpaper]{article}

\usepackage{cvpr}
\usepackage{times}
\usepackage{epsfig}
\usepackage{graphicx}
\usepackage{amsmath}
\usepackage{amssymb}
\usepackage{subfigure}
\usepackage{xcolor}
\usepackage{footnote}
\usepackage{makecell}

\usepackage[utf8]{inputenc}
\usepackage[english]{babel}

\usepackage{algorithm}
\usepackage{algpseudocode}

\usepackage{subfigure}
\usepackage{xcolor}

\usepackage{amsthm}

\newtheorem{theorem}{Theorem}[section]

\newtheorem{lemma}[theorem]{Lemma}

\def\x{{\mathbf x}}

\def\I{{\mathbf I}}

\def\cc{{\mathbf c}}

\def\T{{\mathcal T}}
\def\F{{\mathcal F}}
\def\R{{\mathbb R}}

\newcommand\norm[1]{\left\lVert#1\right\rVert}

\usepackage[pagebackref=true,breaklinks=true,letterpaper=true,colorlinks,bookmarks=false]{hyperref}

 \cvprfinalcopy % *** Uncomment this line for the final submission

\def\cvprPaperID{5310} % *** Enter the CVPR Paper ID here

\ifcvprfinal\fi
\begin{document}

%%%%%%%%% TITLE
\title{A Theoretically Sound Upper Bound on the Triplet Loss for Improving the Efficiency of Deep Distance Metric Learning}

\author{Thanh-Toan Do$^{1}$, Toan Tran$^{2}$, Ian Reid$^{2}$, Vijay Kumar$^{3}$, Tuan Hoang$^{4}$, Gustavo Carneiro$^{2}$\\
{\small $^{1}$University of Liverpool}
{\small $^{2}$University of Adelaide} 
{\small $^{3}$PARC}
{\small $^{4}$Singapore University of Technology and Design}}

\maketitle

\begin{abstract}
We propose a method that substantially improves the efficiency of deep distance metric learning based on the optimization of the triplet loss function.
One epoch of such training process based on a na\"ive optimization of the triplet loss function has a run-time complexity $\mathcal{O}(N^3)$, where $N$ is the number of training samples.
Such optimization scales poorly, and the most common approach proposed to address this high complexity issue is based on sub-sampling the set of triplets needed for the training process. 
Another approach explored in the field relies on an ad-hoc linearization (in terms of $N$) of the triplet loss that introduces class centroids, which must be optimized using the whole training set for each mini-batch -- this means that a na\"ive implementation of this approach has run-time complexity $\mathcal{O}(N^2)$. This complexity issue is usually mitigated with poor, but computationally cheap, approximate centroid optimization methods.
In this paper, we first propose a solid theory on the linearization of the triplet loss with the use of class centroids, where the main conclusion is that our new linear loss represents a tight upper-bound to the triplet loss.  Furthermore, based on the theory above, we propose a training algorithm that no longer requires the centroid optimization step, which means that our approach is the first in the field with a guaranteed linear run-time complexity.
We show that the training of deep distance metric learning methods using the proposed upper-bound is substantially faster than triplet-based methods, while producing competitive retrieval accuracy results on benchmark datasets (CUB-200-2011 and CAR196).

\end{abstract}

\section{Introduction}

\begin{figure}
	\centering
	\includegraphics[scale=0.38]{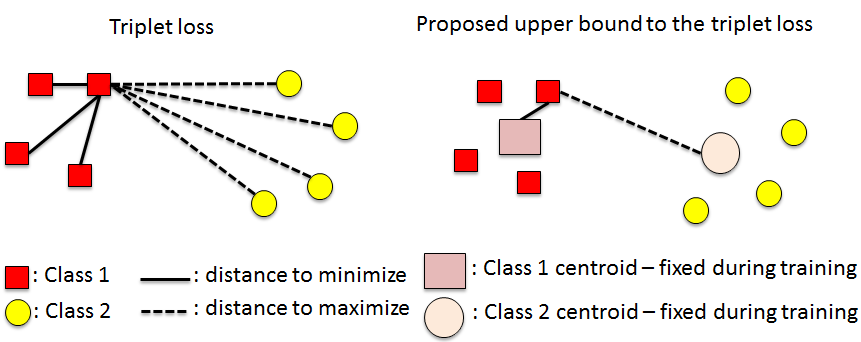}
	\caption{Left: the triplet loss requires the calculation of the loss over $8 \times 3 \times 4 = 96$ elements (cubic complexity in the number of samples). Right: the proposed upper bound to the triplet loss only requires the loss calculation for $8 \times 2=16$ elements (linear complexity in the number of samples and centroids).  In addition, note that the class centroids are fixed during training, overcoming the expensive centroid optimization step, with quadratic complexity, found in similar approaches~\cite{guerriero2018deepncm,DBLP:conf/nips/SnellSZ17,wang2017normface,DBLP:conf/eccv/WenZL016}.}
	\label{fig:motivation}
\end{figure}

Deep distance metric learning (DML) aims at training a deep learning model that transforms training samples into feature embeddings that are close together for samples that belong to the same class and far apart for samples from different classes~\cite{dosovitskiy2014discriminative,han2015matchnet,smart,global,masci2014descriptor,shrivastava2016training,pascal,Simonyanpami14,lifted,histogram,wohlhart2015learning,ZagoruykoCVPR15}.  The use of DML is advantageous compared to more traditional classification models because DML does not rely on a classification layer that imposes strong constraints on the type of problems that the trained model can handle.  For instance, if a model is trained to classify 1000 classes, then the addition of the 1001$^{st}$ class will force the design of a new model structure that incorporates that extra classification node.  In addition, DML requires no model structure update, which means that the learned DML model can simply be fine-tuned with the training samples for the new class.  Therefore, DML is an interesting approach to be used in learning problems that can be continuously updated, such as open-world~\cite{bendale2015towards} and life-long learning problems~\cite{thrun2012learning}.

One of the most common optimization functions explored in DML is the triplet loss, which comprises two terms: 1) a term that minimizes the distance between pairs of feature embeddings belonging to the same class, and 2) another term that maximizes the distance between pairs of feature embeddings from different classes.  The training process based on this triplet loss has run-time complexity $\mathcal{O}(N^3/C)$ per epoch, with $N$ representing the number of samples and $C<N$, the number of classes.
Given that training sets are becoming increasingly larger,  
DML training based on such triplet loss is computationally challenging, and a great deal of work has been focused on the reduction of this complexity without affecting much the effectiveness of DML training.
One of the main ideas explored is the design of mechanisms that select representative subsets of the $\mathcal{O}(N^3)$ triplet samples --- some examples of this line of research are: hard or semi-hard triplet mining~\cite{facenet,pascal,wang2014learning} or smart triplet mining~\cite{smart}.
Unfortunately, these methods still present high computational  complexity, with a worst case training complexity of $O(N^2)$.  Moreover, these approaches also present the issue that the subset of selected triplets  may not be representative enough, increasing the risk of overfitting such subsets.  The mitigation of this problem is generally achieved with the incorporation of an additional loss term that optimizes a global classification to the whole triplet subset~\cite{smart,global,histogram} in order to regularize the training process.

Another idea explored in the field is the ad-hoc linearization of the triplet loss~\cite{guerriero2018deepncm,DBLP:conf/nips/SnellSZ17,wang2017normface,DBLP:conf/eccv/WenZL016}, consisting of the use of auxiliary class centroids.
The training process consists of two alternating steps: 1) an optimization function that generally pulls embeddings towards their class centroid and pushes embeddings away from all other class centroids (hence, $\mathcal{O}(NC)$); and 2) an optimization of the class centroids using the whole training set after the processing of each mini-batch (hence, $\mathcal{O}(N \times N/B)$, where $B$ represents the number of samples in each mini-batch). Therefore, a na\"ive implementation of this method has run-time complexity proportional to $\mathcal{O}(N^2)$. 

In this paper, we provide a solid theoretical background that fully justifies the linearization of the triplet loss, providing a tight upper bound to be used in the DML training process, which relies on the use of class centroids. In particular, our theory shows that the proposed upper-bound differs from the triplet loss by a value that tends to zero if the distances between centroids are large and these distances are similar to each other, and as the training process progresses. Therefore, our theory guarantees that a minimization of the upper bound is also a minimization of the triplet loss. Furthermore, we derive a training algorithm that no longer requires an optimization of the class centroids, which means that our method is the first approach in the field that guarantees a linear run-time complexity for the triplet loss approximation.  Figure~\ref{fig:motivation} motivates our work. 
We show empirically that the DML training using our proposed loss function is one order of magnitude faster than the training of recently proposed triplet loss based methods. In addition, we show that the models trained with our proposed loss produces competitive retrieval accuracy results on benchmark datasets (CUB-200-2011 and CAR196).

\section{Related Work}
\label{sec:sota}

\paragraph{Classification loss.}

It has been shown that deep networks that are trained for the classification task with a softmax 
loss function can be used to produce useful deep feature embeddings. In particular, in~\cite{genericfea,genericfea0} the authors showed that the features extracted from one of the last layers of the deep classification models~\cite{alexnet,vgg} can be used for new classification tasks, involving classes not used for training. 
In addition, the run-time training complexity is quite efficient: $\mathcal{O}(NC)$, where $N$ and $C<N$ represent the number of training samples and number of classes, respectively. 
However, these approaches rely on cross-entropy loss that tries to pull samples over the classification boundary for the class of that sample, disregarding two important points in DML: 1) how close the point is to the class centroid; and 2) how far the sample is from other class centroids (assuming that a class centroid can be defined to be at the centre of the classification volume for each class in the embedding space).  Current evidence in the field shows that not explicitly addressing these two issues made these approaches not attractive for DML, particularly in regards to the classification accuracy for classes not used for training.

\paragraph{Pairwise loss.}
One approach, explored in ~\cite{siamac,pascal} employs a siamese model trained with a pairwise loss. One of the most studied pairwise losses is the contrastive loss~\cite{contrastive0}, which minimizes the distance between pairs of training samples belonging to same class (i.e., positive pairs) and  maximizes the distance between pairs of training samples from different classes (i.e., negative pairs) as long as this ``negative distance'' is smaller than a margin. 
There are few issues associated with this approach. Firstly, the run-time training complexity is $\mathcal{O}(N^2)$, which makes 
this approach computationally challenging for most modern datasets.  
To mitigate this challenge, mining strategies have been proposed to select a subset of the original $\mathcal{O}(N^2)$ pairwise samples.  Such mining strategies focus on selecting pairs of samples that are considered to be hard to classify by the current model.
For instance, Simo-Serra et al.~\cite{pascal} proposed a method that samples positive pairs and sorts them in descending order with respect to the distance between the two samples in the embedding space.  A similar approach is applied for negative pairs, but the sorting is in ascending order. Then, the top pairs in both lists are used as the training pairs. The second issue is that the margin parameter is not easy to tune because the distances between samples can change significantly during the training process. 
Another issue is the fact that the arbitrary way of sampling pairs of samples described above cannot guarantee that the selected pairs are the most informative to train the model.  The final issue is that the optimization of the positive pairs is independent from the negative pairs, but the optimization should force the distance between positive pairs to be smaller than negative pairs.

\paragraph{Triplet loss.}

The triplet loss addresses the last issue mentioned above~\cite{facenet,smart,DBLP:conf/cvpr/QianJZL15}, and it is defined based on three data points: an anchor point, a positive point (i.e., a point belonging to the same class as the anchor), and a negative point (i.e., a point from a different class of the anchor). The loss will force the positive pair distance plus a margin to be smaller than the negative pair distance. However, similarly to pairwise loss, setting the margin in the triplet loss requires careful tuning. Furthermore, and also similarly to the pairwise loss, the training complexity is quite high with $\mathcal{O}(N^3/C)$, hence several triplet mining strategies have been proposed. For instance, in~\cite{facenet}, the authors proposed a ``semi-hard'' criterion, where a triplet is selected if the negative distance is small (i.e., within the margin) but larger than the positive distance --- this approach reduces the training complexity to $\mathcal{O}(N^3/CM^2)$, where $M < N$ represents the number of mini-batches used for training. In~\cite{smart}, the authors proposed to use fast approximate nearest neighbor search for quickly identifying informative (hard) triplets for training, reducing the training complexity to $\mathcal{O}(N^2)$. In~\cite{tripletproxy}, the mining is replaced by the use of $P < N$ proxies, where a triplet is re-defined to be an anchor point, a similar and a dissimilar proxy -- this reduces the training complexity to $\mathcal{O}(NP^2)$. Movshovitz et al.~\cite{tripletproxy} show that this loss computed with proxies represents an upper bound to the original triplet loss, where this bound gets closer to the original triplet loss as $P \rightarrow N$, which increases the complexity back to $\mathcal{O}(N^3)$. It is worth noting that the idea of using proxies 
and learning the embedding such as minimizing the distances between samples to their proxies has been investigated in~\cite{DBLP:conf/nips/PerrotH15}. However, different from~\cite{tripletproxy} which is a DML approach and is expected to capture the nonlinearity between samples, thanks to the power of the deep model, in~\cite{DBLP:conf/nips/PerrotH15} the authors used precomputed features, and to capture the nonlinearity, they relied on kernelization.
 We note that the approach~\cite{tripletproxy} has a non-linear term that makes it more complex than the $\mathcal{O}(NC)$ complexity of our approach, for $C \approx P$, which is usually the case.  Moreover, the  approach~\cite{tripletproxy} also requires the optimization of the number and locations of proxies~\cite{tripletproxy} during training, while our approach relies on a set of predefined and fixed $C$ centroids.
 
%\vspace{-0.3cm}
\paragraph{Other losses.}

In~\cite{global} the authors proposed a global loss function 
that uses the first and second order statistics of sample distance distribution in the embedding space to allow for robust training of triplet network, but the complexity is still $\mathcal{O}(N^3)$. Ustinova and Lempitsky~\cite{histogram} proposed a histogram loss function that is computed by estimating two distributions of similarities for positive  and negative pairs. Based on the estimated distributions, the loss will compute the probability of a positive pair to have a lower similarity score than a negative pair, where the training complexity is $\mathcal{O}(N^2)$. In~\cite{clustering} the authors proposed a loss which optimizes a global clustering metric (i.e., normalized mutual information). This loss ensures that the score of the ground truth clustering assignment is higher than the score of any other clustering assignment -- this method has complexity $\mathcal{O}(NY^3)$, where $Y<N$ represents the number of clusters. 
Similarly to~\cite{tripletproxy}, this approach has a non-linear term w.r.t. number of clusters, that makes it more complex than the $\mathcal{O}(NC)$ complexity of our approach. In addition, this method also requires to optimize the locations of clusters during training, while our approach relies on a set of predefined and fixed $C$ centroids. 
In~\cite{npair}, the authors proposed the N-pair loss which generalizes the triplet loss by allowing joint comparison among more than one negative example -- the complexity of this method is again $\mathcal{O}(N^3)$. 
More recent works~\cite{opitz2017bier,wang2017deep} proposed the use of ensemble classifiers~\cite{opitz2017bier} and new similarity metrics~\cite{wang2017deep} which can in principle explore the $\mathcal{O}(NC)$ training loss that we propose.  
A relevant alternative method recently proposed in the field is related to an ad-hoc linearization of the triple loss that, differently from our approach, has not been theoretically justified~\cite{guerriero2018deepncm,DBLP:conf/nips/SnellSZ17,wang2017normface,DBLP:conf/eccv/WenZL016}.  In addition, even though these approaches rely on a loss function that has run-time complexity $\mathcal{O}(NC)$, they also need to run an expensive centroid optimization step after processing each mini-batch, which has complexity $\mathcal{O}(N)$.  Assuming that a mini-batch has $N/B$ samples, then the run-time complexity of this approach is $\mathcal{O}(NC + N^2/B)$.  Most of the research developed for these methods are centered on the mitigation of this  $\mathcal{O}(N^2)$ complexity involved in the class centroid optimization.  Interestingly, this step is absent from our proposed approach, which means that our method is the only approach in the field that is guaranteed to have linear run-time complexity.

\section{Discriminative Loss}
\label{sec:disloss}
Assume that the training set is represented by $\T=\{\I_i, y_i\}_{i=1}^{N}$ in which $\I_i \in \mathbb R^{H \times W}$ and $y_i \in \{1,...,C\}$ denote the training image and its class label, respectively. Let $\x_i \in \R^D$ be the feature embedding of $\I_i$, obtained from the deep learning model $\x = f(\I,\theta)$. To control the magnitude of distance between feature embeddings, we assume that $\norm{\x}=1$ (i.e., all points lie on a unit hypersphere\footnote{We use $l_2$ distance in this work.}). From an implementation point of view, this assumption can be guaranteed with the use of a normalization layer. Furthermore, without loss of generalization, let us assume that the dimension of the embedding space equals the number of classes, i.e., $D=C$. Note that if $D \neq C$ we can enforce this assumption by adding a fully connected layer to project features from $D$ dimensions to $C$ dimensions.

\subsection{Discriminative Loss: Upper Bound on the Triplet Loss} 

In order to avoid the cubic complexity (in the number of training points) of the triplet loss and to avoid the complicated hard-negative mining strategies, we propose a new loss function that has linear complexity on $N$, but inherits the property of triplet loss: feature embeddings from the same classes are pulled  together, while feature embeddings from different classes are pushed apart. 

    Assume that we have a set $\mathcal{S} =  \{(i,j) | y_i = y_j,   \}_{i,j \in \{1,...,N \}}$ representing pairs of images $\I_i$ and $\I_j$ belonging to the same class. Let us start with a simplified form of the triplet loss:
\begin{equation}
L_t(\T,\mathcal{S}) = \sum_{(i,j) \in S, (i,k) \notin S, i,j,k \in \{1,...,N\} } \ell_t(\x_i,\x_j, \x_k),
\label{eq:triplet}
\end{equation}
where $\ell_t$ is defined as 
\begin{equation}
\ell_t(\x_i,\x_j, \x_k) = \norm{\x_i - \x_j} - \norm{\x_i - \x_k}.
\label{eq:triplet0}
\end{equation}
Let $\mathcal{C} = \{\cc_m\}_{m=1}^C$, where each $\cc_m \in \R^D$ is an auxiliary vector in the embedding space that can be seen as the ``centroid'' for the $m^{th}$ class (note that as the centroids represent  classes in the embedding space, they should be defined in the same domain with the embedding features, i.e., on the surface of the unit hypersphere). According to the triangle inequality, we have 
\begin{equation}
\norm{\x_i - \x_j} \le \norm{\x_i  - \cc_{y_i}} + \norm{\x_j - \cc_{y_i}},
\label{eq:tri1}
\end{equation}
and
\begin{equation}
\norm{\x_i - \x_k} \ge \norm{\x_i  - \cc_{y_k}} - \norm{\x_k - \cc_{y_k}}.
\label{eq:tri2}
\end{equation}
From (\ref{eq:tri1}) and (\ref{eq:tri2}) we achieve the upper bound for $\ell_t$ as follows:
\begin{equation}
\ell_t(\x_i,\x_j, \x_k) \le \ell_d(\x_i, \x_j, \x_k),
\label{eq:uptriplet0}
\end{equation}
where 
\begin{eqnarray}
\ell_d(\x_i, \x_j, \x_k)&=&\norm{\x_i - \cc_{y_i}} - \norm{\x_i - \cc_{y_k}}\nonumber \\ 
{}&+& \norm{\x_j - \cc_{y_i}} + \norm{\x_k - \cc_{y_k}}
\label{eq:ld}
\end{eqnarray}
From (\ref{eq:triplet}) and (\ref{eq:uptriplet0}), we have 
\begin{equation}
L_t(\T,\mathcal{S}) \le \sum_{(i,j) \in S, (i,k) \notin S, i,j,k \in \{1,...,N\}} \ell_d(\x_i, \x_j, \x_k).
\label{eq:uptriplet1}
\end{equation}
The central idea in the paper is to minimize the upper bound defined in (\ref{eq:uptriplet1}). 
Assume that we have a balanced training problem, where the number of samples in each class is equal (for the imbalanced training, this assumption can be enforced by data augmentation), after some algebraic manipulations, the RHS of (\ref{eq:uptriplet1}) is equal to $L_d(\T,\mathcal{S})$ which is our proposed discriminative loss
\begin{equation}
\footnotesize
L_d(\T,\mathcal{S}) =  G \sum_{i=1}^N \left( \norm{\x_i - \cc_{y_i}} - \frac{1}{3(C-1)} \sum_{m=1,m \neq y_i}^C \norm{\x_i - \cc_{m}} \right)
\label{eq:disloss}
\end{equation}
where the constant $G=3(C-1) \left( \frac{N}{C} - 1\right) \frac{N}{C}$.

Our goal is to minimize $L_d(\T,\mathcal{S})$ which is a discriminative loss that simultaneously pulls samples from the same class close to their centroid and pushes samples far from centroids of different classes. A nice property of $L_d(\T,\mathcal{S})$ is that it is not arbitrary far from $L_t(\T,\mathcal{S})$. The difference between these two losses is well bounded by Lemma~\ref{lemma}.  
\begin{lemma}
\label{lemma}

Assuming that $\norm{\x_i - \cc_{y_i}} \leq \epsilon/2$ (with $\epsilon \geq 0$), $1\le i\le N$. Let $\kappa_{min} = \min_{1\le m,n \le C, m\neq n} \norm{\cc_{m} - \cc_{n}}$ and $\kappa_{max} = \max_{1\le m,n \le C, m\neq n} \norm{\cc_{m} - \cc_{n}}$, 
then \\
$0 \le L_d(\T,\mathcal{S}) - L_t(\T,\mathcal{S}) \le H \left( \kappa_{max} - \kappa_{min} + 3\epsilon\right)$, where the constant $H=\left(\frac{N}{C}-1\right)N\left( N-\frac{N}{C}\right)$ is the number of all possible triplets.
\end{lemma}
 
\begin{proof}
The proof is provided in the Appendix. 
\end{proof}

From Lemma~\ref{lemma}, $L_d$ will approach $L_t$ when  $\epsilon \to 0$ and $\kappa_{min} \rightarrow \kappa_{max}$. (i) Note from \eqref{eq:disloss} that $\epsilon$ will decrease because the discriminative loss pulls samples from the same class close to their corresponding centroid.  (ii) In addition, we can enforce that $\kappa_{min} \approx \kappa_{max}$ by fixing the centroids before the training starts, such that they are as far as possible from each other and the distances between them are similar.  
Therefore, with the observations (i) and (ii), we can expect that $\epsilon \rightarrow 0$ and $\kappa_{max} - \kappa_{min} \rightarrow 0$, which implies a tight bound from Lemma~\ref{lemma}.  We discuss methods to generate centroid locations below.

\subsection{Centroid Generation}
\label{sec:centroid_gen}
From the observations above, we should have centroids on the surface of the unit hypersphere such that they are as far as possible from each other and the distances between them are as similar as possible. Mathematically, we want to maximize the minimum distance between $C$ centroids -- this problem is known as the Tammes problem~\cite{tammes}. Let $\F$ be the surface of the hypersphere, we want to solve the following optimization:
\begin{eqnarray}
\min_{ \{ \mathbf{c}_m \}_{m=1}^C} && -w \nonumber \\ 
{} && \hspace{-4.5em}{\textrm{s.t.}}\ \ \ \ w  \leq  \norm{\cc_m - \cc_n},  \ 1 \le m, n \le C, m \neq n \label{eq:tammes}\\ 
{} && \cc_m \in \F,  \ m = 1,...,C \nonumber
\end{eqnarray}

Unfortunately, it is not possible to solve (\ref{eq:tammes}) analytically in general~\cite{sphere}. We may solve it as an optimization problem. However, this optimization will involve $\mathcal{O}(C^2)$ constraints, hence the problem is still computationally \textit{hard} to solve for large $C$~\cite{sphere1}. To overcome this challenge, we propose two heuristics to generate the centroids. 

\paragraph{One-hot centroids.} Inspired by the softmax loss, we define the centroids as vertices of a convex polyhedron in which each vertex is a one-hot vector, i.e., the centroid of the $m^{th}$ class is the standard basis vector in $m^{th}$ direction of the Cartesian coordinate system. With this configuration, each centroid is orthogonal to each other and the distance between each pair of centroids is $\sqrt{2}$. 

%\vspace{-0.3cm}
\vspace{0.2cm}
\paragraph{K-means centroids.} We first uniformly generate a large set of points on the surface of the unit hypersphere. We then run K-means clustering to group points into $C$ clusters. The unit normalized cluster centroids will be used as centroids $\{ \mathbf{c}_m \}_{m=1}^C$. Note that the process of uniformly generating points on the surface of the unit  hypersphere is not difficult. According to~\cite{sphere2}, for each point, we generate each of its dimension independently with the standard normal distribution. Then we unit normalize the point to project it to the surface of the hypersphere.

\subsection{Discussion on the Discriminative Loss}
\label{subsec:diss}

\paragraph{Training complexity.} 
Table~\ref{tab:complexity} compares the asymptotic  training complexity of several DML methods, including our proposed discriminative loss (Discriminative) in terms of the number of training samples $N$, number of mini-batches $M$, size of mini-batch $B$, number of classes $C$, number of proxies $P$~\cite{tripletproxy} and number of clusters $Y$~\cite{clustering}.  
It is clear from (\ref{eq:disloss}) that our proposed discriminative loss has linear run-time complexity (in terms of both $N$ and $C$), analogous to the softmax loss~\cite{genericfea}. The methods that optimize an approximate triplet loss, which have linear complexity in $N$ are represented by ``centroids''~\cite{guerriero2018deepncm,DBLP:conf/nips/SnellSZ17,wang2017normface,DBLP:conf/eccv/WenZL016} in Table~\ref{tab:complexity}, but note in the table that the optimization of the centroids must be performed after processing each mini-batch, which increases the complexity of the approach to be square in $N$. 
Most of the research in the   ``centroids'' approach goes into the reduction of the complexity in the optimization of class centroids  
with the design of poor, but computationally cheap, approximate methods. For example, in~\cite{DBLP:conf/eccv/WenZL016}, instead of updating the centroids with respect to the entire training set, the authors perform the update based on mini-batch. This leads to a linear complexity in $N$. However by updating centroids based on mini-batch, a  small  batch  size  (e.g.  due  to  a  large  network  structure, which is likely) may cause a poor approximation of the real centroids. In the worst case, when not all classes are present in a batch, some centroids are even not be updated.
Interestingly, the centroid update step is absent from our proposed approach.

There are other DML methods that are linear in $N$: clustering~\cite{clustering} with $\mathcal{O}(NY^3)$ and triplet+proxy~\cite{tripletproxy} with $\mathcal{O}(NP^2)$.  There are two advantages of our approach compared to these two methods in terms of training complexity: 1) our discriminative loss is linear not only in terms of the dominant variable $N$, but also with respect to the auxiliary variable $C < N$ (where in general $C \approx P,Y$); and 2) in our work, the number of centroids and their positions are fixed before the training process starts (as explained in Sec.~\ref{sec:centroid_gen}), hence there is no need to optimize the number and positions of centroids during training --- this contrasts with the fact that the number and positions of clusters and proxies need to be optimized in~\cite{clustering} and~\cite{tripletproxy}.

\begin{table}[!t]
%\centering
\vspace{0.2cm}
\begin{center}
\resizebox{\columnwidth}{!}{
\begin{tabular}{c c c c}
\hline
 Softmax~\cite{genericfea} &Pair-na\"ive &Trip.-na\"ive &Trip.-hard~\cite{facenet}\\
$\mathcal{O}(NC)$ &$\mathcal{O}(N^2)$  &$\mathcal{O}(N^3/C)$ &$\mathcal{O}(N^3/(M^2C))$ \\
\hline
 Trip.-smart~\cite{smart} & Trip.-cluster~\cite{clustering} & Trip.-proxy~\cite{tripletproxy} & Centroids~\cite{guerriero2018deepncm,DBLP:conf/nips/SnellSZ17,wang2017normface,DBLP:conf/eccv/WenZL016} \\
$\mathcal{O}(N^2)$ &$\mathcal{O}(NY^3)$ &$\mathcal{O}(NP^2)$ &$\mathcal{O}(NC + N^2/B)$ \\
\hline
\multicolumn{4}{c}{Discriminative} \\
\multicolumn{4}{c}{$\mathcal{O}(NC)$} \\
\hline
\end{tabular}}
%\vspace{0.5em}
\end{center}
\caption{Run-time training complexity of various DML approaches and our proposed discriminative loss in terms of the number of training samples $N$, number of mini-batches $M$, size of mini-batch $B$, number of classes $C$, number of proxies $P$~\cite{tripletproxy} and number of clusters $Y$~\cite{clustering}.}
\label{tab:complexity}
\end{table}

\paragraph{Simplicity. } The discriminative loss only involves the calculation of Euclidean distance between the embedding features and the centroids. Hence it is straightforward to implement and integrate into any deep learning models to be trained with the standard back-propagation. Furthermore, different from most of the traditional DML losses such as pairwise loss, triplet loss, and their improved versions~\cite{facenet,npair,global,tripletproxy,smart},  
the discriminative loss does not require setting margins, mining triplets, and optimizing the number and locations of centroids during training. This reduction in the number of hyper-parameters makes the training simpler and improves the performance (compared to standard triplet methods), as showed in the experiments.

\vspace{0.1cm}
\section{Experiments}
\label{sec:exp}\vspace{0.2cm}
\subsection{Dataset and Evaluation Metric}

We conduct our experiments on two public benchmark datasets that are commonly used to evaluate DML methods, where we follow the standard experimental protocol for both datasets~\cite{smart,npair,clustering,lifted}. 
The \textbf{CUB-200-2011} dataset~\cite{cub} contains 200 species of birds with 11,788 images, where the first 100 species with 5,864 images are used for training and the remaining 100 species with 5,924 images are used for testing. The \textbf{CAR196} dataset~\cite{car} contains 196 car classes with 16,185 images, where the first 98 classes with 8,054 images are used for training and the remaining 98 classes with 8,131 images are used for testing. We report the K nearest neighbor retrieval accuracy using the Recall@K metric. We also report the clustering quality using the normalized mutual information (NMI) score~\cite{nmi}. 
\vspace{0.1cm}
\subsection{Network Architecture and Implementation Details}
\label{subsec:architecture}
For all experiments in Sections~\ref{subsec:study} and \ref{subsec:sota}, we initialize the network with the pre-trained GoogLeNet~\cite{googlenet} -- this is also a standard practice in the comparison between DML approaches~\cite{smart,npair,clustering,lifted}. We then add two randomly initialized fully connected layers. The first layer has $256$ nodes, which is the commonly used embedding dimension in previous works, and the second layer has $C$ nodes. 
We train the network for a maximum of 40 epochs. 
For the last two layers, we start with an initial learning rate of 0.1 and gradually decrease it by a factor of 2 every 5 epochs. Following~\cite{lifted}, all GoogLeNet layers are fine-tuned with the learning rate that is ten times smaller than the learning rate of the last two layers. 
The weight decay and the batch size are set to 0.0005 and 128, respectively in all experiments. As normally done in previous works, random cropping and random horizontal flipping are used when training. 

\begin{table*}[!t]
   \centering
   %\footnotesize
    \begin{center}
    \begin{tabular}{c c c c c| c c c c} 
    \hline 
    &\multicolumn{4}{c|}{{CUB-200-2011}} &\multicolumn{4}{c}{{CAR196}}\\
&R@1 &R@2 &R@4 &R@8 &R@1 &R@2 &R@4 &R@8\\  
Last layer &49.49 &62.32 &72.52 &81.57        &65.32 &76.44 &84.10 &89.84\\
Second to last layer &51.43 &64.23 &74.31 &82.83 &68.31 &78.21 &85.22 &91.18\\
\hline
    \end{tabular}
    \end{center}
    %\vspace{-0.2cm}
    \caption{Performance of embedding features from the last two layers.} 
    \label{tab:layer}
\end{table*}

\begin{table*}[!t]
\centering
\begin{center}
\begin{tabular}{c c c c c c c c c}
\hline
&\multicolumn{8}{c}{{CUB-200-2011}} \\
&R@1 &R@2 &R@4 &R@8 &min dist. &max dist. &mean dist. &std dist.\\
one-hot cent. &51.43 &64.23 &74.31 &82.83 &$\sqrt[]{2}$ &$\sqrt[]{2}$ &$\sqrt[]{2}$ &0 \\
K-means cent. &50.75 &63.54 &73.26 &82.36 &1.21 &1.63 &1.418 &0.061 \\
&\multicolumn{8}{c}{{CAR196}} \\
one-hot cent. &68.31 &78.21 &85.22 &91.18 &$\sqrt[]{2}$ &$\sqrt[]{2}$ &$\sqrt[]{2}$ &0 \\
K-means cent. &66.93 &76.74 &83.80 &90.37 &1.18 &1.65 &1.416 &0.066 \\
\hline
\end{tabular}
\end{center}
%\vspace{-0.2cm}
\caption{Performance of two different centroid generation strategies and the statistics of distances (dist.) between centroids.}
\label{tab:centroid}
\end{table*}

\begin{table*}[!t]
	\centering
	\begin{center}
		\begin{tabular}{c c c c c c} 
			\hline
			&NMI &R@1 &R@2 &R@4 &R@8 \\ \hline
			SoftMax &57.21 &48.34 &60.16 &71.21 &80.30 \\
			Semi-hard~\cite{facenet} &55.38 &42.59 &55.03 &66.44 &77.23 \\
			Lifted structure~\cite{lifted} &56.50 &43.57 &56.55 &68.59 &79.63 \\
			N-pair~\cite{npair} &57.24 &45.37 &58.41 &69.51 &79.49 \\
			Triplet+Global~\cite{global} &58.61 &49.04 &60.97 &72.33 &81.85 \\
			Clustering~\cite{clustering} &59.23 &48.18 &61.44 &71.83 &81.92 \\
			Triplet+smart mining~\cite{smart} &59.90 &49.78 &62.34 &74.05 &\textbf{83.31} \\
			Triplet+proxy~\cite{tripletproxy} &59.53 &49.21 &61.90 &67.90 &72.40 \\
			Histogram~\cite{histogram} &- &50.26 &61.91 &72.63 &82.36 \\
			Discriminative &\textbf{59.92} &\textbf{51.43} &\textbf{64.23} &\textbf{74.31} &{82.83} \\ \hline
		\end{tabular}
	\end{center}
	%\vspace{-0.2cm}
	\caption{Clustering and Recall performance on the CUB-200-2011 dataset.} 
	\label{tab:cub}
\end{table*}

\begin{table*}[!t]
	\centering
	\begin{center}
		\begin{tabular}{c c c c c c} 
			\hline
			&NMI &R@1 &R@2 &R@4 &R@8 \\ \hline
			SoftMax &58.38 &62.39 &72.96 &80.86 &87.37 \\
			Semi-hard~\cite{facenet} &53.35 &51.54 &63.78 &73.52 &82.41 \\
			Lifted structure~\cite{lifted} &56.88 &52.98 &65.70 &76.01 &84.27 \\
			N-pair~\cite{npair} &57.79 &53.90 &66.76 &77.75 &86.35 \\
			Triplet+Global~\cite{global} &58.20 &61.41 &72.51 &81.75 &88.39 \\
			Clustering~\cite{clustering} &59.04 &58.11 &70.64 &80.27 &87.81 \\
			Triplet+smart mining~\cite{smart} &59.50 &64.65 &76.20 &84.23 &90.19 \\
			Triplet+proxy~\cite{tripletproxy} &\textbf{64.90} &\textbf{73.22} &\textbf{82.42} &\textbf{86.36} &88.68 \\
			Histogram~\cite{histogram} &- &54.34 &66.72 &77.22 &85.17 \\
			Discriminative &59.71 &68.31 &78.21 &85.22 &\textbf{91.18} \\ \hline
		\end{tabular}
	\end{center}
	%\vspace{-0.2cm}
	\caption{Clustering and Recall performance on the CAR196 dataset.} 
	\label{tab:car}
\end{table*}

\subsection{Ablation Study}
\label{subsec:study}
\vspace{0.1cm}
\paragraph{Effect of features from different layers.} In this experiment we evaluate the embedding features from the last two fully connected layers with dimensions $256$ and $C$ ($C=100$ for CUB-200-2011 and $C = 98$ for CAR196).  These results are based on the one-hot centroid generation strategy, but note that the same evidence was produced with the K-means centroid generation.
The results in Table~\ref{tab:layer} show that the features from the second to last layer produce better generalization for unseen classes than those from the last layer. The possible reason is that the features from the last layer may be too specific to the set of training classes. Hence for tasks on unseen classes, the features from the second to last layer produce better performance. The same observation is also found in~\cite{genericfea} (although Razavian et al.~\cite{genericfea} experimented with AlexNet~\cite{alexnet}). Hereafter, we only use the features from the second to last fully connected layer for the remaining experiments. Note that this also allows for a fair comparison between our work and previous approaches in terms of feature extraction complexity because these other approaches also use the feature embeddings extracted from the same layer.

\paragraph{Effect of centroid generation method.}
In this section, we evaluate the two proposed centroid generation methods, explained in Sec.~\ref{sec:centroid_gen}, where the hypersphere for the K-means approach has $C$ dimensions. 
The comparative performances and  statistics of distances between centroids are shown in Table~\ref{tab:centroid}.  The results show that there is not a significant difference in performance between the two centroid generation methods. In the worst case, K-means is 1.5\% worse than one-hot on CAR196 dataset while on CUB-200-2011 dataset, these two methods are comparable. Hereafter, we only use one-hot centroid generation strategy for all remaining experiments. 

According to Table~\ref{tab:centroid}, we note that the difference between the minimum and the maximum distances between centroids is quite small for K-means and $0$ for the one-hot centroid generation methods.  This is an important fact for the triplet loss bound in the Lemma~\ref{lemma}, where the smaller this difference, the tighter the bound to the triplet loss.

\subsection{Comparison with Other Methods}
\label{subsec:sota}
We compare our method to the baseline DML methods that have reported results on the standard datasets CUB-200-2011 and CAR196: the  \textbf{softmax} loss, the \textbf{triplet} loss with \textbf {semi-hard} negative mining~\cite{facenet}, the \textbf{lifted structured} loss~\cite{lifted}, the \textbf{N-pair}~\cite{npair} loss, the \textbf{clustering} loss~\cite{clustering}, the \textbf{triplet} combined with \textbf{global} loss~\cite{global}, the \textbf{histogram} loss~\cite{histogram}, the \textbf{triplet with proxies}~\cite{tripletproxy} loss, \textbf{triplet with smart mining} ~\cite{smart} loss which uses the fast nearest neighbor search for mining triplets. 

Tables~\ref{tab:cub} and~\ref{tab:car} show the recall and NMI scores for the baseline methods and our approach (\textbf{Discriminative}). 
The results on Tables~\ref{tab:cub} and~\ref{tab:car} show that for the NMI metric, most triplet-based methods achieve comparable results, except for \textbf{Triplet+proxy}~\cite{tripletproxy} which has a 5.2\% gain over the second best \textbf{Discriminative} on the CAR196 dataset. 
Under Recall@K metric, the \textbf{Discriminative} improves over most of methods that are based on triplet, (e.g., \textbf{Semi-hard}~\cite{facenet})  or  generalization of triplet (e.g., \textbf{N-pair}~\cite{npair}, \textbf{Triplet+Global}~\cite{global}). Compared to the softmax loss, although both discriminative loss and softmax loss have the same complexity, \textbf{Discriminative} improves over \textbf{Softmax} by a large margin for all measures on both datasets. This suggests that the  discriminative loss is more suitable for DML than the  softmax loss.

\textbf{Discriminative} also compares favorably with the recent \textbf{triplet+smart mining} method~\cite{smart}, i.e., on the CAR196 dataset, \textbf{Discriminative} has 3.6\% improvement in R@1 over the \textbf{triplet+smart mining}. Compared to the recent \textbf{Triplet+proxy} on the CUB-200-2011 dataset, \textbf{Discriminative} shows better results at all ranks of $K$, where larger improvements are observed at larger $K$, i.e.,  \textbf{Discriminative} has 10.4\% (14.4\% relative) improvement in R@8 over \textbf{Triplet+proxy}. On the CAR196 dataset, \textbf{Triplet+proxy} outperforms the \textbf{Discriminative} at low values of $K$, i.e., \textbf{Triplet+proxy} has 4.9\% (7.2\% relative) higher accuracy than \textbf{Discriminative} at R@1. However, for increasing values of $K$, the improvement of \textbf{Triplet+proxy} decreases, and \textbf{Discriminative} achieves a higher accuracy than \textbf{Triplet+proxy} at R@8.

We are aware that there are other triplet-based methods that achieve better performance on CUB and CAR196 datasets~\cite{Duan_2018_CVPR,Lin_2018_ECCV,opitz2017bier,wang2017deep,yuan2016hard}. Table~\ref{tab:sota-triplet} presents their results.  However, it is important to note that although these methods use the triplet loss, they rely on additional techniques to boost their accuracy.  For instance, Yuan et al.~\cite{yuan2016hard} used cascaded embedding to ensemble a set of models; Opitz et al.~\cite{opitz2017bier} relied on boosting to combine different learners; Wang et al.~\cite{wang2017deep} combined angular loss with N-pair loss~\cite{npair} to boost performance; Duan et al.~\cite{Duan_2018_CVPR} and Lin et al.~\cite{Lin_2018_ECCV} used generative adversarial network (GAN) to generate synthetic training samples. We note that these techniques can in principle replace their triplet loss by our discriminative loss to improve training efficiency. However, this is out of scope of this paper, but we consider this to be future work.
\begin{table}
\footnotesize
\centering
 \begin{center}
\begin{tabular}{c c c c c c c}
\hline
		 		&\cite{yuan2016hard} &\cite{opitz2017bier} &\cite{wang2017deep} &\cite{Duan_2018_CVPR} &\cite{Lin_2018_ECCV} &Discrim.\\ \hline
CUB-200-2011    &53.6 &55.3 &54.7 &52.7 &52.7 &51.4\\ \hline
CAR196 			&73.7 &78.0 &71.4 &75.1 &82.0 &68.3\\ \hline
\end{tabular}
\end{center}
\caption{R@1 comparison to the state of the art on CUB-200-2011 and CAR196 datasets. Although all these methods relied on the triplet loss, they also use additional techniques specifically designed to boost classification accuracy. 
}
\label{tab:sota-triplet}
\end{table}
\vspace{-0.4cm}
\paragraph{Training time complexity.}
To demonstrate the efficiency of the proposed method, we also compare the empirical training time of the  proposed discriminative loss to other triplet-based methods, i.e., \textbf{Semi-hard}~\cite{facenet} and \textbf{triplet with smart mining}~\cite{smart}. 
All methods were tested on the same machine and we use the default configurations of ~\cite{facenet} and~\cite{smart}.

 \begin{table}[!t]
 \centering
 \footnotesize
 \begin{center}
 \begin{tabular}{c c c c }
 \hline
 &Semi-hard~\cite{facenet} &Triplet &Discrim. \\
 &						   &+smart mining~\cite{smart} &\\
 CUB-200-2011 &660 & 680 &54\\
 CAR196 & 1200 & 1240 &73 \\
 \hline
 \end{tabular}
 \end{center}
 \caption{Training time in minutes between different methods. The CUB-200-2011 dataset consists of $5864$ training images with 100 classes, and the CAR196 dataset consists of $8054$ training images with 98 classes.}
 \label{tab:trainingtime}
 \end{table}

The results in Table~\ref{tab:trainingtime} show that the training time of the proposed methods (\textbf{Discrim.}) is around 13 and 17 times faster than the recent state-of-the-art \textbf{triplet with smart mining}~\cite{smart} on CUB and CAR196 datasets, respectively. The results also confirm that our loss scales linearly w.r.t. number of training images and number of classes, i.e., $(5864 \times 100)/(8054 \times 98) \approx 54/73$.
 
\subsection{Improving with Different Network Architectures}
As presented in Section~\ref{subsec:diss}, the proposed loss is simple and it is easy to integrate into any deep learning models. To prove the flexibility of the proposed loss, in this section we experiment with the VGG-16 network~\cite{vgg}. Specifically, we apply a max-pooling on the last convolutional layer of VGG to produce a 512-D feature representation. After that, similarly to GoogleNet in Section~\ref{subsec:architecture}, we add two fully connected layers whose dimensions are $256$ and $C$. The outputs of the second to last layer are used as embedding features. Table~\ref{tab:vgg} presents the results when using our discriminative loss with VGG network. 
From Tables~\ref{tab:cub},~\ref{tab:car} and~\ref{tab:vgg}, we can see that using discriminative loss with VGG network significantly boosts the performance on both datasets, e.g., at $R@1$, it improves over GoogleNet $6.3\%$ and $9.8\%$ for CUB-200-2011 and CAR196, respectively. 
 
\begin{table}
\footnotesize
\centering
 \begin{center}
\begin{tabular}{c c c c c c}
\hline
		 		&NMI &R@1 &R@2 &R@4 &R@8 \\ \hline
CUB-200-2011    &61.49 &57.74 &68.46 &78.07 &85.40  \\ \hline
CAR196 			&62.14 &78.15 &85.70 &90.71 &94.21 \\ \hline
\end{tabular}
\end{center}
\caption{Clustering and recall performance when using VGG-16 network with discriminative loss.}
\label{tab:vgg}
\end{table}

We note that using other advance network architectures such as Inception~\cite{inceptionV2}, ResNet~\cite{resnet} rather than GoogleNet~\cite{googlenet}, VGG~\cite{vgg}, may give performance boost as showed in recent works~\cite{Ge_2018_ECCV,DBLP:conf/iccv/ManmathaWSK17,DBLP:conf/eccv/XuanSP18}. 
However, that is not the focus of this paper. Our work targets on developing a linear complexity loss that approximates the triplet loss but offers faster training process with a similar accuracy to the triplet loss.

\section{Conclusion}
\label{sec:concl}
In this paper we propose the first deep distance metric learning method that approximates the triplet loss and is guaranteed to have linear training complexity. Our  proposed discriminative loss is based on an upper bound to the triplet loss, and we theoretically show that this bound is tight depending on the distribution of class centroids. We propose two methods to generate class centroids that enforce that their distribution guarantees the tightness of the bound.  
The experiments on two benchmark datasets show that in terms of retrieval accuracy, the proposed method is competitive while its training time is one order of magnitude faster than triplet-based methods. 
Consequently, this paper proposes the most efficient DML approach in the field, with competitive DML retrieval performance.

\section*{Acknowledgments} This work was partially supported by the Australian Research Council project (DP180103232).

\appendix
\section*{Appendix: Proof for the Lemma~\ref{lemma}}
\begin{proof}
The lower bound, i.e.,  $0 \le L_d(\T,\mathcal{S}) - L_t(\T,\mathcal{S})$ is straightforward by (\ref{eq:uptriplet0}). Here we prove the upper bound. 

By the assumption, for any $\x_i$ and its centroid $\cc_{y_i}$ we have 
\begin{equation}
\norm{\x_i - \cc_{y_i}} \leq \epsilon/2
\label{p1}
\end{equation}
By using the triangle inequality and (\ref{p1}), for any $\x_i$ and the centroids $\cc_{y_i}$, $\cc_{y_k}$ where $\cc_{y_k} \neq \cc_{y_i}$, we have
\begin{equation}
\norm{\x_i - \cc_{y_k}} \ge \norm{\cc_{y_i} - \cc_{y_k}} - \norm{\x_i - \cc_{y_i}} \ge  \kappa_{min} - \epsilon/2
\label{p2}
\end{equation}

From (\ref{eq:ld}), (\ref{p1}), (\ref{p2}), for any triplet $(\x_i, \x_j, \x_k)$  we have
\begin{equation}
 \ell_d \le -\kappa_{min} + 2\epsilon
\label{eq:ldbound}
\end{equation}

By using the norm property and (\ref{p1}), for any pair of $\x_i$ and $\x_k$ that are not same class, we have
\begin{eqnarray}
\norm{\x_i - \x_k} &=&\norm{\x_i - \cc_{y_i} + \cc_{y_i} - \cc_{y_k} + \cc_{y_k} - \x_k} \nonumber \\ 
{}&\leq& \kappa_{max} + \epsilon
\label{p4}
\end{eqnarray}
From (\ref{eq:triplet0}), (\ref{p4}), for any triplet $(\x_i, \x_j, \x_k)$  we have
\begin{equation}
 -\ell_t \le \kappa_{max} + \epsilon
\label{eq:ltbound}
\end{equation}

The upper bound for $L_d - L_t$ is achieved by adding (\ref{eq:ltbound}) and (\ref{eq:ldbound}) over all possible triplets.
\end{proof}

{\small
\bibliographystyle{ieee}
\bibliography{refs}

\begin{thebibliography}{10}\itemsep=-1pt

\bibitem{genericfea}
H.~Azizpour, A.~S. Razavian, J.~Sullivan, A.~Maki, and S.~Carlsson.
\newblock From generic to specific deep representations for visual recognition.
\newblock In {\em {CVPR} Workshops}, 2015.

\bibitem{bendale2015towards}
A.~Bendale and T.~Boult.
\newblock Towards open world recognition.
\newblock In {\em CVPR}, 2015.

\bibitem{contrastive0}
S.~Chopra, R.~Hadsell, and Y.~LeCun.
\newblock Learning a similarity metric discriminatively, with application to
  face verification.
\newblock In {\em CVPR}, 2005.

\bibitem{dosovitskiy2014discriminative}
A.~Dosovitskiy, J.~T. Springenberg, M.~Riedmiller, and T.~Brox.
\newblock Discriminative unsupervised feature learning with convolutional
  neural networks.
\newblock In {\em NIPS}, 2014.

\bibitem{Duan_2018_CVPR}
Y.~Duan, W.~Zheng, X.~Lin, J.~Lu, and J.~Zhou.
\newblock Deep adversarial metric learning.
\newblock In {\em CVPR}, 2018.

\bibitem{Ge_2018_ECCV}
W.~Ge.
\newblock Deep metric learning with hierarchical triplet loss.
\newblock In {\em ECCV}, 2018.

\bibitem{guerriero2018deepncm}
S.~Guerriero, B.~Caputo, and T.~Mensink.
\newblock Deepncm: Deep nearest class mean classifiers.
\newblock In {\em ICLR Workshop}, 2018.

\bibitem{han2015matchnet}
X.~Han, T.~Leung, Y.~Jia, R.~Sukthankar, and A.~C. Berg.
\newblock Matchnet: Unifying feature and metric learning for patch-based
  matching.
\newblock In {\em CVPR}, 2015.

\bibitem{smart}
B.~Harwood, B.~G.~V. Kumar, G.~Carneiro, I.~D. Reid, and T.~Drummond.
\newblock Smart mining for deep metric learning.
\newblock In {\em ICCV}, 2017.

\bibitem{resnet}
K.~He, X.~Zhang, S.~Ren, and J.~Sun.
\newblock Deep residual learning for image recognition.
\newblock In {\em {CVPR}}, 2016.

\bibitem{sphere1}
H.~Huang, P.~M. Pardalos, and Z.~Shen.
\newblock A point balance algorithm for the spherical code problem.
\newblock {\em Journal of Global Optimization}, 19(4):329--344, 2001.

\bibitem{car}
J.~Krause, M.~Stark, J.~Deng, and L.~Fei-Fei.
\newblock 3d object representations for fine-grained categorization.
\newblock In {\em ICCV Workshops}, 2013.

\bibitem{alexnet}
A.~Krizhevsky, I.~Sutskever, and G.~E. Hinton.
\newblock Imagenet classification with deep convolutional neural networks.
\newblock In {\em {NIPS}}, 2012.

\bibitem{global}
B.~G.~V. Kumar, G.~Carneiro, and I.~Reid.
\newblock Learning local image descriptors with deep siamese and triplet
  convolutional networks by minimising global loss functions.
\newblock In {\em {CVPR}}, 2016.

\bibitem{sphere}
P.~Leopardi.
\newblock Distributing points on the sphere: Partitions, separation, quadrature
  and energy.
\newblock {\em PhD thesis, School of Mathematics and Statistics, the University
  of New South Wales}, 2006.

\bibitem{Lin_2018_ECCV}
X.~Lin, Y.~Duan, Q.~Dong, J.~Lu, and J.~Zhou.
\newblock Deep variational metric learning.
\newblock In {\em ECCV}, 2018.

\bibitem{DBLP:conf/iccv/ManmathaWSK17}
R.~Manmatha, C.~Wu, A.~J. Smola, and P.~Kr{\"{a}}henb{\"{u}}hl.
\newblock Sampling matters in deep embedding learning.
\newblock In {\em ICCV}, 2017.

\bibitem{nmi}
C.~D. Manning, P.~Raghavan, and H.~Sch\"{u}tze.
\newblock {\em Introduction to Information Retrieval}.
\newblock Cambridge University Press, 2008.

\bibitem{sphere2}
G.~Marsaglia.
\newblock Choosing a point from the surface of a sphere.
\newblock {\em The Annals of Mathematical Statistics}, 43(2):645--646, 1972.

\bibitem{masci2014descriptor}
J.~Masci, D.~Migliore, M.~M. Bronstein, and J.~Schmidhuber.
\newblock Descriptor learning for omnidirectional image matching.
\newblock In {\em Registration and Recognition in Images and Videos}, pages
  49--62. Springer, 2014.

\bibitem{tripletproxy}
Y.~Movshovitz{-}Attias, A.~Toshev, T.~K. Leung, S.~Ioffe, and S.~Singh.
\newblock No fuss distance metric learning using proxies.
\newblock In {\em {ICCV}}, 2017.

\bibitem{opitz2017bier}
M.~Opitz, G.~Waltner, H.~Possegger, and H.~Bischof.
\newblock Bier-boosting independent embeddings robustly.
\newblock In {\em ICCV}, 2017.

\bibitem{DBLP:conf/nips/PerrotH15}
M.~Perrot and A.~Habrard.
\newblock Regressive virtual metric learning.
\newblock In {\em NIPS}, 2015.

\bibitem{DBLP:conf/cvpr/QianJZL15}
Q.~Qian, R.~Jin, S.~Zhu, and Y.~Lin.
\newblock Fine-grained visual categorization via multi-stage metric learning.
\newblock In {\em {CVPR}}, 2015.

\bibitem{siamac}
F.~Radenovic, G.~Tolias, and O.~Chum.
\newblock {CNN} image retrieval learns from bow: Unsupervised fine-tuning with
  hard examples.
\newblock In {\em ECCV}, 2016.

\bibitem{genericfea0}
A.~S. Razavian, H.~Azizpour, J.~Sullivan, and S.~Carlsson.
\newblock {CNN} features off-the-shelf: An astounding baseline for recognition.
\newblock In {\em {CVPR} Workshops}, 2014.

\bibitem{facenet}
F.~Schroff, D.~Kalenichenko, and J.~Philbin.
\newblock Facenet: {A} unified embedding for face recognition and clustering.
\newblock In {\em {CVPR}}, 2015.

\bibitem{shrivastava2016training}
A.~Shrivastava, A.~Gupta, and R.~Girshick.
\newblock Training region-based object detectors with online hard example
  mining.
\newblock In {\em CVPR}, 2016.

\bibitem{pascal}
E.~Simo{-}Serra, E.~Trulls, L.~Ferraz, I.~Kokkinos, P.~Fua, and
  F.~Moreno{-}Noguer.
\newblock Discriminative learning of deep convolutional feature point
  descriptors.
\newblock In {\em ICCV}, 2015.

\bibitem{Simonyanpami14}
K.~Simonyan, A.~Vedaldi, and A.~Zisserman.
\newblock Learning local feature descriptors using convex optimisation.
\newblock {\em TPAMI}, 2014.

\bibitem{vgg}
K.~Simonyan and A.~Zisserman.
\newblock Very deep convolutional networks for large-scale image recognition.
\newblock {\em CoRR}, 2014.

\bibitem{DBLP:conf/nips/SnellSZ17}
J.~Snell, K.~Swersky, and R.~S. Zemel.
\newblock Prototypical networks for few-shot learning.
\newblock In {\em NIPS}, 2017.

\bibitem{npair}
K.~Sohn.
\newblock Improved deep metric learning with multi-class n-pair loss objective.
\newblock In {\em NIPS}, 2016.

\bibitem{clustering}
H.~O. Song, S.~Jegelka, V.~Rathod, and K.~Murphy.
\newblock Deep metric learning via facility location.
\newblock In {\em CVPR}, 2017.

\bibitem{lifted}
H.~O. Song, Y.~Xiang, S.~Jegelka, and S.~Savarese.
\newblock Deep metric learning via lifted structured feature embedding.
\newblock In {\em CVPR}, 2016.

\bibitem{googlenet}
C.~Szegedy, W.~Liu, Y.~Jia, P.~Sermanet, S.~E. Reed, D.~Anguelov, D.~Erhan,
  V.~Vanhoucke, and A.~Rabinovich.
\newblock Going deeper with convolutions.
\newblock In {\em CVPR}, 2015.

\bibitem{inceptionV2}
C.~Szegedy, V.~Vanhoucke, S.~Ioffe, J.~Shlens, and Z.~Wojna.
\newblock Rethinking the inception architecture for computer vision.
\newblock In {\em {CVPR}}, 2016.

\bibitem{tammes}
P.~M.~L. Tammes.
\newblock On the origin of number and arrangements of the places of exit on the
  surface of pollen-grains.
\newblock {\em Recueil des Travaux Botaniques N{\'e}erlandais}, pages 1--84,
  1930.

\bibitem{thrun2012learning}
S.~Thrun and L.~Pratt.
\newblock {\em Learning to learn}.
\newblock Springer Science \& Business Media, 2012.

\bibitem{histogram}
E.~Ustinova and V.~S. Lempitsky.
\newblock Learning deep embeddings with histogram loss.
\newblock In {\em {NIPS}}, 2016.

\bibitem{cub}
C.~Wah, S.~Branson, P.~Welinder, P.~Perona, and S.~Belongie.
\newblock The {Caltech-UCSD} birds-200-2011 dataset.
\newblock 2011.

\bibitem{wang2017normface}
F.~Wang, X.~Xiang, J.~Cheng, and A.~L. Yuille.
\newblock Normface: L2 hypersphere embedding for face verification.
\newblock In {\em ACM MM}, 2017.

\bibitem{wang2014learning}
J.~Wang, Y.~Song, T.~Leung, C.~Rosenberg, J.~Wang, J.~Philbin, B.~Chen, and
  Y.~Wu.
\newblock Learning fine-grained image similarity with deep ranking.
\newblock In {\em CVPR}, 2014.

\bibitem{wang2017deep}
J.~Wang, F.~Zhou, S.~Wen, X.~Liu, and Y.~Lin.
\newblock Deep metric learning with angular loss.
\newblock In {\em ICCV}, 2017.

\bibitem{DBLP:conf/eccv/WenZL016}
Y.~Wen, K.~Zhang, Z.~Li, and Y.~Qiao.
\newblock A discriminative feature learning approach for deep face recognition.
\newblock In {\em ECCV}, 2016.

\bibitem{wohlhart2015learning}
P.~Wohlhart and V.~Lepetit.
\newblock Learning descriptors for object recognition and 3d pose estimation.
\newblock In {\em CVPR}, 2015.

\bibitem{DBLP:conf/eccv/XuanSP18}
H.~Xuan, R.~Souvenir, and R.~Pless.
\newblock Deep randomized ensembles for metric learning.
\newblock In {\em ECCV}, 2018.

\bibitem{yuan2016hard}
Y.~Yuan, K.~Yang, and C.~Zhang.
\newblock Hard-aware deeply cascaded embedding.
\newblock In {\em ICCV}, 2017.

\bibitem{ZagoruykoCVPR15}
S.~Zagoruyko and N.~Komodakis.
\newblock Learning to compare image patches via convolutional neural networks.
\newblock In {\em CVPR}, 2015.

\end{thebibliography}
}

\end{document}